\documentclass{article}

\usepackage{microtype}
\usepackage{graphicx}
\usepackage{subfigure}
\usepackage{booktabs} 

\usepackage{amsbsy, amsfonts, amsgen, amsmath, amsopn, amssymb, amstext,
amsxtra, bezier, color, enumerate, graphicx, latexsym, verbatim,
pictexwd, supertabular, url, dsfont,leftidx, mathrsfs, appendix, setspace, amsthm}

\usepackage{hyperref}

\newtheorem{theorem}{Theorem}
\newtheorem{question}{Question}
\newtheorem{proposition}{Proposition}
\newtheorem{lemma}{Lemma}
\newtheorem{corollary}{Corollary}
\newtheorem{remark}{Remark}

\renewcommand{\tilde}{\widetilde}

\newcommand{\vct}{\boldsymbol }

\definecolor{DSgray}{cmyk}{0,1,0,0}

\def \bP {\mathbb{P}}
\def \bE {\mathbb{E}}
\def \bR {\mathbb{R}}

\def\1{\vct {1}}

\newcommand{\naturals}{\mathbb{N}}

\def\ud{\mathrm{d}}

\usepackage[accepted]{icml2021}

\icmltitlerunning{Adversarial Combinatorial Bandits with General Reward Function}

\begin{document}

\twocolumn[
\icmltitle{Adversarial Combinatorial Bandits with General Non-linear Reward Functions}



\icmlsetsymbol{equal}{*}

\begin{icmlauthorlist}
\icmlauthor{Xi Chen}{equal,nyu}
\icmlauthor{Yanjun Han}{equal,stanford}
\icmlauthor{Yining Wang}{equal,uf}
\end{icmlauthorlist}

\icmlaffiliation{nyu}{Stern School of Business, New York University, New York, NY 10012, USA}
\icmlaffiliation{stanford}{Department of Electrical Engineering, Stanford University, Stanford, CA 94305, USA}
\icmlaffiliation{uf}{Warrington College of Business, University of Florida, Gainesville, FL 32611, USA}

\icmlcorrespondingauthor{Yining Wang}{yining.wang@warrington.ufl.edu}

\icmlkeywords{combinatorial bandit, assortment optimization, minimax analysis}

\vskip 0.3in
]



\printAffiliationsAndNotice{\icmlEqualContribution} 

\begin{abstract}
In this paper we study the adversarial combinatorial bandit with a known non-linear reward function, extending existing work on adversarial linear combinatorial bandit. {The adversarial combinatorial bandit with general non-linear reward is an important open problem in bandit literature, and it is still unclear whether there is a significant gap from the case of linear reward, stochastic bandit, or semi-bandit feedback.} We show that, with $N$ arms and subsets of $K$ arms being chosen at each of $T$ time periods, the minimax optimal regret is $\widetilde\Theta_{d}(\sqrt{N^d T})$ if the reward function is a $d$-degree polynomial with $d< K$, and $\Theta_K(\sqrt{N^K T})$ if the reward function is not a low-degree polynomial. {Both bounds are significantly different from the bound $O(\sqrt{\mathrm{poly}(N,K)T})$ for the linear case, which suggests that there is a fundamental gap between the linear and non-linear reward structures.} Our result also finds applications to adversarial assortment optimization problem in online recommendation. We show that in the worst-case of adversarial assortment problem, the optimal algorithm must treat each individual $\binom{N}{K}$ assortment as independent.
\end{abstract}

\section{Introduction}\label{sec:intro}

In this paper we study the \emph{combinatorial bandit} problem, which is a natural extension to the multi-armed bandit problem \citep{auer1995gambling}
and has applications to online advertising, online shortest paths and many other practical problems \citep{cesa2012combinatorial,chen2013combinatorial,chen2016combinatorial,chen2016combinatorial_b,wang2018thompson}.
In the adversarial combinatorial bandit setting, there are $T$ time periods and $N$ arms. At the beginning of each time period $t$, 
an adaptive adversary chooses a reward vector $v_t=(v_{t1},\cdots,v_{tN})\in [0,1]^N$ not revealed to the algorithm.
The algorithm chooses a subset $S_t\subseteq[N]$ consisting of exactly $K\leq N$ distinct arms (i.e., $|S_t|=K$).
The algorithm then receives a \emph{bandit} feedback $r_t\in[0,1]$ satisfying
\begin{equation}
\mathbb E[r_t|S_t,v_t] = g\left(\sum_{i\in S_t}v_{ti}\right),
\label{eq:defn-rt}
\end{equation}
where $g:\mathbb R^+\to[0,1]$ is a known link function.
The objective is to minimize the \emph{regret} of the algorithm compared against a stationary benchmark, defined as
\begin{equation}
\max_{|S^\star|=K} \mathbb E\left[\sum_{t=1}^T R(S^\star,v_t)-R(S_t,v_t)\right],
\label{eq:defn-regret}
\end{equation}
where $R(S,v_t) := g(\sum_{i\in S}v_i)$ and $\{S_t\}_{t=1}^T$ are the subsets outputted by a regret minimization algorithm.

As far as we know, all existing works on \emph{adversarial} combinatorial bandit studied only the case when the link function is linear (i.e., $g(x)=cx$) \citep{cesa2012combinatorial,bubeck2012towards,audibert2014regret}.
While there have been research on combinatorial bandits with general link functions, such results are established exclusively for the \emph{stochastic} setting, in which the reward vectors $\{v_t\}_{t=1}^T$ do \emph{not} change over time \citep{rusmevichientong2010dynamic,agarwal2018regret,agrawal2019mnl}, and most of them further assume a \emph{semi-bandit} feedback where noisy observations of all $v_{ti}$ with $i\in S_t$ are available \citep{combes2015combinatorial,kveton2015tight,chen2016combinatorial,chen2018dynamic,chen2018optimal}.
This motivates the following natural question:
\begin{question}
For adversarial combinatorial bandit with a non-linear link function $g(\cdot)$ and bandit feedback, is it possible to achieve 
$\widetilde O(\sqrt{\mathrm{poly}(N,K)T})$ regret?
\label{ques:main}
\end{question}

Note that in adversarial combinatorial bandit with linear link function, \emph{or} stochastic combinatorial (semi-)bandit with general link functions, the $\widetilde  O(\sqrt{\mathrm{poly}(N,K)T})$ regret targeted in Question \ref{ques:main} can be attained \citep{bubeck2012towards,combes2015combinatorial,kveton2015tight,chen2016combinatorial,agrawal2019mnl}.
The question also has important practical motivations beyond theoretical/mathematical reasoning, 
because many interesting applications of combinatorial bandit involve non-linear link functions,
such as online assortment optimization with a multinomial logit (MNL) model, which corresponds to a link function of $g(x)=x/(1+x)$.  Please see more discussions Sec.~\ref{subsec:intro-mnl}.

\subsection{Our results}\label{subsec:our-results}

Below is an informal statement of our main result, as a summary of Theorem \ref{thm:main-result} later in the paper.
\begin{corollary}[Informal]
Fix an arbitrary, known reward function $g:\mathbb R^+\to[0,1]$.
If $g$ is a $d$-degree polynomial for some $d<K$, then the optimal regret is 
$
\widetilde\Theta_{g,d,K}\big(\sqrt{N^d T}\big).
$

Otherwise, if $g$ is either not a polynomial or a polynomial of degree at least $K$, the optimal regret is
$
\Theta_{g,K}\big(\sqrt{N^K T}\big).
$
\label{cor:main-result-informal}
\end{corollary}

The results in Corollary \ref{cor:main-result-informal} easily cover the linear link function case $g(x)=x$,
with $d=1$ and the optimal regret being $\widetilde\Theta(\sqrt{NT})$ \citep{bubeck2012towards}.
On the other hand, Corollary \ref{cor:main-result-informal} shows that when $g$ is not a polynomial,
no algorithm can achieve a regret better than $O(\sqrt{N^K T})$.
This shows that when $g$ is a general non-linear reward function, the $\binom{N}{K}$ subsets of the $N$ arms can only be treated as  ``independent''
and it is information-theoretically impossible for any algorithm to exploit correlation between subsets of arms to achieve a significantly smaller regret.

\subsection{Dynamic assortment optimization}\label{subsec:intro-mnl}

\emph{Dynamic assortment optimization} is a key problem in revenue management and online recommendation, which naturally serves as a motivation
for the generalized combinatorial bandit problem studied in this paper.
In the standard setup of dynamic assortment optimization \citep{agrawal2019mnl}, there are $N$ substitutable products,
each associated with a known profit parameter $p_i\in[0,1]$ and an unknown mean  utility parameter $v_i\in[0,1]$.
At each time, the seller offers an \emph{assortment} (i.e., the recommended set of products) $S_t\subseteq[N]$ of size $K$, e.g., there are $K$ display spots of recommended products on an Amazon webpage. Then the customer either purchases one 
of the products being offered (i.e., $i_t\in S_t$) or leaves without making any purchase (i.e., $i_t=0$). 
The choice behavior of the customer is governed by the well-known MultiNomial-Logit (MNL) model from economics \citep{Train2009}:
\begin{equation}
\mathbb P[i_t=i|S_t, v] = \frac{v_i}{v_0+\sum_{j\in S_t}v_j}, \;\;\;\;\;\forall i\in S_t\cup\{0\},
\label{eq:defn-mnl}
\end{equation}
with the definition that $v_0 := 1$, where $v_0$ denote the utility of no-purchase.
The objective for the seller or retailer is to maximize the expected profit/revenue $$R(S, v) = \sum_{i\in S} p_i\mathbb P[i|S,v] = \frac{\sum_{i\in S}p_i v_i}{v_0+\sum_{i\in S}v_i}.$$
Note also that, in the adversarial setting, the mean utility vector $v_t=\{v_{ti}\}_{i=1}^N$ will be different for each time period $t=1,2,\cdots,T$, and will be selected by an adaptive adversary. The regret is then defined as \eqref{eq:defn-regret}. 

{Let us first consider a special case, where all the products have the profit one (i.e., $p_i\equiv 1$) and only binary purchase/no-purchase actions are observable. That is, one only observes a binary reward at time $t$,  $r_t=\vct 1\{i_t\in S_t\}$, which indicates whether there is a purchase.} 
Then the dynamic assortment optimization question reduces to the generalized (adversarial) combinatorial bandit problem formulated in Eqs.~(\ref{eq:defn-rt},\ref{eq:defn-regret}) with the link function $g(x)=x/(1+x)$. 
Since $g(x)=x/(1+x)$ is clearly not a polynomial, Corollary \ref{cor:main-result-informal} shows that $\Theta(\sqrt{N^K T})$ should be the optimal regret.
The following corollary extends this to the general case of dynamic assortment optimization, where different products can have different profit parameters.
\begin{corollary}
Consider the dynamic assortment optimization question with known profit parameters $\{p_i\}_{i=1}^N\subseteq[0,1]$ and
unknown mean utility parameters $\{v_{ti}\}_{t,i=1}^{T,N}\subseteq[0,1]$ chosen by an adaptive adversary.
Then there exists an algorithm with regret upper bounded $O_K(\sqrt{N^K T})$.
\label{cor:main-result-informal-mnl-ub}
\end{corollary}
The next corollary, on the other hand, shows that the $O(\sqrt{N^K T})$ regret in not improvable, even with a richer non-binary feedback and if all products have the same profits parameter $p_i\equiv 1$.
\begin{corollary}
Suppose $p_i\equiv 1$. There exists an adaptive adversary that chooses $\{v_{ti}\}_{t,i=1}^{T,N}\subseteq[0,1]$, such that for any algorithm,
the regret is lower bounded by $\Omega_K(\sqrt{N^K T})$.
\label{cor:main-result-informal-mnl-lb}
\end{corollary}

Both corollaries are consequences of Proposition \ref{prop:ub-1} and Lemma \ref{lem:main-result-mnl-lb} later in the paper.

\subsection{Proof techniques}

As we shall see later in this paper, the upper bounds of $\widetilde O_{g,K,d}(\sqrt{N^d T})$ or $O_{g,K}(\sqrt{N^K T})$ in Corollary \ref{cor:main-result-informal}
are relatively easier to establish, via reduction to known adversarial multi-armed or linear bandit algorithms.
The key challenge is to establish corresponding $\Omega(\sqrt{N^d T})$ and $\Omega(\sqrt{N^K T})$ \emph{lower bounds} in Corollary \ref{cor:main-result-informal}.

In this section we give a high-level sketch of the key ideas in our lower bound proof. For simplicity we consider only the case when $g(\cdot)$ is not a polynomial function.
The key insight is to prove the existence of a distribution $\mu$ on $v\in[0,1]^n$, such that for any $S\subseteq[n]$, $|S|=K$, the following holds on {the choice distribution $\mathbb P(\cdot|S,v)$} with $\bP_0 \neq \bP_1$:
\begin{equation}
\mathbb E_{v\sim \mu}[\mathbb P(\cdot|S,v)] \equiv \left\{\begin{array}{ll}
\mathbb P_0,& \text{if }S=S^\star,\\
\mathbb P_1,& \text{if }S\neq S^\star.\end{array}\right.
\label{eq:lb-intuition}
\end{equation}
Intuitively, Eq.~(\ref{eq:lb-intuition}) shows that, no information is gained unless an algorithm \emph{exactly} guesses the optimal subset $S^\star$,
even if the subset $S_t$ produced by the algorithm only differ by a single element from $S^\star$.
Since there are $\binom{N}{K}$ different subsets, 
the question of guessing the optimal subset $S^\star$ exactly correct is similar to locating the best arm of a multi-armed bandit question
with $\binom{N}{K}=\Theta_K(N^K)$ arms, which incurs a regret of $\Omega_K(\sqrt{N^K T})$.

To gain deeper insights into the construction of $v\sim\mu$ that satisfies Eq.~(\ref{eq:lb-intuition}), it is instructive to 
consider some simpler bandit settings in which a small regret can be achieved and understand why the construction of $\mu$ does not apply there.
\begin{itemize}
\item The first setting is dynamic assortment optimization under the stationary setting, in which the $\{v_t\}_{t=1}^T$ vectors remain the same for all $T$ periods.
The results of \cite{agrawal2019mnl} achieve $\widetilde O(\sqrt{NT})$ regret in this setting.
In this setting, the vector $v$ is deterministic and fixed, and therefore the laws $\mathbb P(\cdot|S,v)$ and $\mathbb P(\cdot|S',v)$ must be correlated as long as $S\cap S'\neq\emptyset$.
This means that Eq.~(\ref{eq:lb-intuition}) cannot be possibly satisfied, with every subset $S\neq S^\star$ revealing no information about $S^\star$.

\item The second setting is the adversarial combinatorial bandit with a linear link function $g(x)=cx$, for which an $\widetilde O_K(\sqrt{NT})$ regret is attainable \citep{bubeck2012towards,combes2015combinatorial}. When $g$ is linear, the expectation of the mixture distribution $\bE_{v\sim \mu}[\mathbb P(\cdot|S,v)]$ is
$$
\mathbb E_{v\sim\mu}[R(S,v)] = g(\langle \vct 1_S, \mathbb E_{\mu}[v] \rangle), 
$$
where $\vct 1_S\in\{0,1\}^n$ is the indicator vector of the subset $S\subseteq[N]$.
Clearly, this is impossible to achieve \eqref{eq:lb-intuition} as there is no vector $w\in \bR^N$ satisfying that $\langle \vct 1_S, w \rangle$ is constant for all $S\neq S^\star$ and $\langle \vct 1_S, w \rangle\neq \langle \vct 1_{S^\star}, w \rangle$. 

\item The third setting is a special stochastic combinatorial bandit, where $v\sim \mu$ is random but there exists a total ordering of the stochastic dominance relation among the components $(\mu_1,\cdots,\mu_N)$ of $\mu=\prod_{i=1}^N \mu_i$, and an increasing $g$. In this setting, it was shown in \cite{agarwal2018regret} that a regret of $\widetilde{O}_{K}(N^{1/3}T^{2/3})$ can be achieved, and the stochastic dominance requirement implies that once an element of $[N]\backslash S^\star$ is replaced by an element of $S^\star$, the expected reward must increase. Therefore, \eqref{eq:lb-intuition} cannot hold in this scenario either. 
\end{itemize}

\subsection{Other related works}

Combinatorial bandit is a classical question in machine learning and has been extensively studied under the settings of stochastic semi-bandits \citep{chen2013combinatorial,combes2015combinatorial,kveton2015tight,chen2016combinatorial,chen2016combinatorial_b,wang2018thompson,merlis2019batch,merlis2020tight}, stochastic bandits \citep{agarwal2018regret,rejwan2020top,kuroki2020polynomial}, and adversarial linear bandits \citep{cesa2012combinatorial,bubeck2012towards,audibert2014regret,combes2015combinatorial}.
In the above-mentioned works, either the reward link function $g(\cdot)$ is linear, or the model is stochastic (stationary). 

There is another line of research on dynamic assortment optimization with the multinomial logit model, which is a form of combinatorial bandit with general reward functions \citep{rusmevichientong2010dynamic,agrawal2017thompson,chen2018dynamic,chen2018optimal,chen2018note,chen2019robust,agrawal2019mnl}.
All of these works are carried out under the stochastic setting, with the exception of \citep{chen2019robust} which studied an $\varepsilon$-contamination model and obtained a regret upper bound $\widetilde O(\sqrt{NT}+\varepsilon T)$.
Clearly, with the adversarial setting in this paper ($\varepsilon=1$) the regret bound in \citep{chen2019robust} becomes linear in $T$ and thus meaningless.

\subsection{Notations}
For a multi-index $\alpha\in \mathbb{N}^d$, let $|\alpha| = \sum_{i=1}^d \alpha_i$, and $D^\alpha f = \partial^{|\alpha|} f / \prod_{i=1}^d \partial x_i^{\alpha_i}$ for a $d$-variate function $f$. For $m\in\mathbb{N}$ and interval $I$, let $C^m(I)$ be the set of $m$-times continuously differentiable functions on $I$. For two probability distributions $P$ and $Q$, let $\mathrm{TV}(P,Q)=\frac{1}{2}\int |dP-dQ|$ and $D_{\mathrm{KL}}(P\|Q) = \int dP\log(dP/dQ)$ be the total variation (TV) distance and the Kullback--Leibler (KL) divergence, respectively. We adopt the standard asymptotic notation: for two non-negative sequences $\{a_n\}$ and $\{b_n\}$, we use the notation $a_n=O_c(b_n)$ to denote that $a_n\le Cb_n$ for all $n$ and constant $C<\infty$ depending only on $c$, $a_n = \Omega_c(b_n)$ to denote $b_n = O_c(a_n)$, and $a_n = \Theta_c(b_n)$ to denote both $a_n = O_c(b_n)$ and $a_n = \Omega_c(b_n)$. We also use $\widetilde{O}(\cdot), \widetilde{\Omega}(\cdot), \widetilde{\Theta}(\cdot)$ to denote the respective meanings up to a multiplicative poly-logarithmic factor in $(N,T)$.

\section{Problem formulation and results}

Suppose there are $N$ arms, $T$ time periods and a known reward function $g:\mathbb R_+\to[0,1]$.
At each time period $t$, the algorithm outputs a subset $S_t\subseteq[N]$, $|S_t|=K$ and receives 
a binary bandit feedback $r_t\in\{0,1\}$. Note that the binary feedback structure can be significantly relaxed for the purpose of upper bounds,
as discussed in Sec.~\ref{sec:upper-bound}.
Let $\mathcal F_t=\{S_\tau,r_\tau\}_{\tau\leq t}$ be the filtration of observable statistics at time $t$,
and $\mathcal V_t = \{v_\tau\}_{\tau\leq t}$ be the filtration of unobservable reward vectors.
Let also $\mathcal A$ be an unknown adversary and $\pi$ be an admissible policy.
The reward dynamics are modeled as follows:
\begin{eqnarray*}
v_t &\sim& \mathcal A(\mathcal F_{t-1}, \mathcal V_{t-1});\\
S_t &\sim& \pi(\mathcal F_{t-1});\\
r_t &\sim& \textstyle \mathrm{Bernoulli}(g(\sum_{i\in S_t}v_{ti})).
\end{eqnarray*}
For any $g(\cdot),N,K,T$, the \emph{minimax regret} $\mathfrak{R}(g,N,K,T)$ is defined as
\begin{multline}
\mathfrak{R}(g,N,K,T) \\
:= \inf_\pi \sup_{\mathcal A}\max_{|S^\star|=K}\mathbb E\left[\sum_{t=1}^T R(S^\star,v_t)-R(S_t,v_t)\right],
\label{eq:defn-minimax}
\end{multline}
where $R(S,v_t) = g(\sum_{i\in S}v_{ti})$, and the expectation is taken with respect to both the bandit algorithm $\pi$
and the adaptive adversary $\mathcal A$.

The following theorem is a rigorous statement of the main result of this paper.
\begin{theorem}
Fix function $g:\mathbb R^+\to[0,1]$ that is $K$-times continuously differentiable on $(0,K)$.
If $g$ is a polynomial with degree $d\in[1,K)$, then there exist constants $0<c_{g,d,K}\leq C_{g,d,K}<\infty$ depending only on $g,d,K$ such that,
for every $N\ge K$ and $T\geq 1$, it holds that
\begin{align*}
c_{g,d,K} \le \frac{\mathfrak{R}(g,N,K,T)}{\min\{T, \sqrt{N^dT}\}} \le C_{g,d,K}\sqrt{\log N}. 
\end{align*}
Furthermore, if $g$ is a polynomial with degree at least $K$ or not a polynomial, then there exist constants $0<c_{g,K}\leq C_{g,K}<\infty$ depending only on $g,K$ such that,
for every $N\ge K$ and $T\geq 1$, it holds that
$$
c_{g,K}\leq \frac{\mathfrak{R}(g,N,K,T)}{\min\{T,\sqrt{N^K T}\}}\leq C_{g,K}.
$$
\label{thm:main-result}
\end{theorem}
\begin{remark}
Based on Propositions \ref{prop:ub-1} and \ref{prop:ub-2} later, the hidden dependence of the constants $C_{g,d,K}, C_{g,K}$ on $(g,d,K)$ is $O(1)$ and $O(\sqrt{K})$, respectively. However, since our lower bound relies on an existential result (cf. Lemma \ref{lem.existential}), the hidden dependence of constants $c_{g,d,K}$ and $c_{g,K}$ is unknown. It is an outstanding open question to characterize an explicit dependence on $(g,d,K)$ in the lower bound. 
\end{remark}

The results in Theorem \ref{thm:main-result} cover the linear reward case of $g(x)=cx$ via $d=1$ and a regret of $\widetilde\Theta_K(\min\{T,\sqrt{NT}\})$,
which matches the existing results on adversarial linear combinatorial bandits.
On the other hand, the results for general non-polynomial reward functions $g(\cdot)$ are quite negative, 
with a $\Theta_{g,K}(\min\{T,\sqrt{N^K T}\})$ regret showing that all the $\binom{N}{K}$ subsets are essentially independent and the bandit algorithm
cannot hope to exploit correlation between overlapping subsets like in the linear case.
Finally, the reward function $g(\cdot)$ being a low-degree polynomial interpolates between the linear case and the general case,
with a regret of $\widetilde\Theta_{g,d,K}(\min\{T,\sqrt{N^d T}\})$ for $d\in(1,K)$ between $\widetilde\Theta_{K}(\min\{T,\sqrt{NT}\})$ and $\Theta_{g,K}(\min\{T,\sqrt{N^K T}\})$.

In the rest of this section we sketch the proofs of Theorem \ref{thm:main-result} by studying the upper bounds and lower bounds separately. We will also adapt the proof of Theorem \ref{thm:main-result} to cover the more general dynamic assortment optimization model described in Sec.~\ref{subsec:intro-mnl}.

\subsection{Upper bounds}\label{sec:upper-bound}

We first prove the $O_K(\min\{T,\sqrt{N^K T}\})$ regret upper bound for general link functions.
In fact, we state the following result that is much more general than Theorem \ref{thm:main-result}.
\begin{proposition}
Suppose at each time $t$ the adversary $\mathcal A$ could choose an arbitrary combinatorial reward model $\{\mathbb P_t(\cdot|S)\}_{S\subseteq[N], |S|=K}$, 
and an arbitrary bandit feedback $\{b_t\}_{t=1}^T$. 
Let also $r_t(b_t)\in [0,1]$ denote the reward as a function of $b_t$.
There exists a bandit algorithm $\pi$ and a universal constant $C<\infty$ such that for any $N\geq K$, $T\geq 1$, 
\begin{multline*}
\sup_{\mathcal A}\max_{|S^\star|=K}\mathbb E\left[\sum_{t=1}^T \mathbb E(r_t(b_t) | S^\star, \mathbb P_t)-\mathbb E(r_t(b_t)|S_t,\mathbb P_t)\right]\\
\leq C\min\{T,\sqrt{N^K T}\}.
\end{multline*}
\label{prop:ub-1}
\end{proposition}
\vspace{-1cm}
\begin{proof}
For each subset $S\subseteq[N]$ of size $K$, let $j_S$ be a constructed arm and $r_t(b_{t,j_S})$ be the bandit reward feedback at time $t$ if the arm $j_S$ is pulled (i.e., subset $S$ is selected).
This reduces the problem to an adversarial multi-armed bandit problem with $\binom{N}{K}$ independent arms.
Applying the Implicitly Normalized Forecaster (INF) algorithm from \citep{audibert2009minimax}, we have the regret upper bound
$O(\min\{T,\sqrt{\binom{N}{K}T}\})=O_K(\min\{T,\sqrt{N^K T}\})$. 
\end{proof}
\vspace{-0.3cm}
We remark that Proposition \ref{prop:ub-1} is more general and contains the $C_{g,K}\min\{T,\sqrt{N^K T}\}$ upper bound in Theorem \ref{thm:main-result} as a special case. 
By considering the feedback model $b_t\in S_t\cup\{0\}$ and $r_t(b_t)=\sum_{i\in S_t}p_i\vct 1\{b_t=i\}$, Proposition \ref{prop:ub-1} also covers the dynamic assortment optimization model described in Sec.~\ref{subsec:intro-mnl} and Corollary \ref{cor:main-result-informal-mnl-ub}.

We next establish the $O(\min\{T,\sqrt{N^d T\log N}\})$ upper bound for polynomial reward functions.
\begin{proposition}
Fix a known $d$-degree polynomial $g(x)=a_dx^d + a_{d-1}x^{d-1}+\cdots + a_1 x + a_0$.
Suppose at each time $t$, conditioned on the selected subset $S_t\subseteq[N]$ of size $K$, 
the bandit feedback $r_t$ is supported on an arbitrary bounded set not necessarily $\{0,1\}$, such that $\mathbb E[r_t|S_t,v_t] = g(\sum_{i\in S_t}v_t)$.
Then there exists a bandit algorithm $\pi$ such that, for any adaptive adversary $\mathcal A$, 
\begin{multline*}
\max_{|S^\star|=K}\mathbb E\left[\sum_{t=1}^T R(S^\star,v_t)-R(S_t,v_t)\right]\\
 \leq C_{g,d,K}\min\{T,\sqrt{N^d T\log N}\},
\end{multline*}
where $R(S,v_t)=g(\sum_{i\in S}v_{ti})$, and $C_{g,d,K}=O(\sqrt{K})$.
\label{prop:ub-2}
\end{proposition}
\begin{proof}
For any $n$-dimensional vector $x\in\mathbb R^n$, let $x^{\otimes d}=(x_{i_1}x_{i_2}\cdots x_{i_d})_{i_1,\cdots,i_d=1}^n\in\mathbb R^{n^d}$ be the order-$d$ tensorization of $x$. It is easy to verify that, for any $0\leq k\leq d$ and $S_t\subseteq[N]$,
$
\textstyle
(\sum_{i\in S_t}v_{ti})^k = \langle v_t,\vct 1_{S_t}\rangle^k = \langle v_t^{\otimes k}, \vct 1_{S_t}^{\otimes k}\rangle,
$
where $\vct 1_{S_t}\in\{0,1\}^n$ is the indicator vector of $S_t$. Hence, 
$$
R(S_t,v_t) = \sum_{k=0}^d a_k\langle v_t,\vct 1_{S_t}\rangle^k = \sum_{k=0}^d a_k\langle v_t^{\otimes k},\vct 1_{S_t}^{\otimes k}\rangle.
$$
Define $\tilde v_t\in\mathbb R^{n^d}$ as 
$$
\tilde v_t := \sum_{k=0}^d a_k \underbrace{v_t\otimes \cdots\otimes v_t}_{k\text{ times}}\otimes \underbrace{\frac{\vct 1}{K}\otimes\cdots\otimes \frac{\vct 1}{K}}_{d-k \text{ times}}.
$$
As $\langle \vct 1, \vct 1_{S_t}\rangle =K$, it is easy to verify that, for every $S_t\subseteq[N]$, $R(S_t, v_t) = \mathbb E[r_t|S_t,v_t] = \langle \tilde v_t,\vct 1_{S_t}^{\otimes d}\rangle$.

With this transformation, the problem reduces to adversarial linear bandit with dimension $D=n^d$ and fixed action space 
$\mathcal{A}=\{\vct 1_{S}^{\otimes d}\}_{S\subseteq[N], |S|=K}$ with $|\mathcal{A}|=\binom{N}{K}$.
Applying the EXP2 algorithm with John's exploration and the analysis in \citep{audibert2009minimax}, the regret is upper bounded by 
$O(\sqrt{DT\log |\mathcal{A}|}) = O(\sqrt{N^dT K\log N})$, which proves Proposition \ref{prop:ub-2}.
\end{proof}

Propositions \ref{prop:ub-1} and \ref{prop:ub-2} complete the proof of minimax upper bounds in Theorem \ref{thm:main-result}.

\subsection{Lower bounds}

We first prove the following result corresponding to the minimax lower bounds in Theorem \ref{thm:main-result}.
\begin{lemma}
Suppose $r_t\sim\text{Bernoulli}(g(\sum_{i\in S_t}v_{ti}))$ for 
some fixed, known function $g:\mathbb R_+^n\to [0,1]$
that is $K$-times continuously differentiable on $(0,K)$.
If $g$ is a degree-$d$ polynomial with $d<K$, then there exists a constant $c_{g,d,K}>0$ such that for all $N\geq K$, $T\geq 1$,
$$
\mathfrak{R}(g,N,K,T)\geq c_{g,d,K}\min\{T,\sqrt{N^d T}\}.
$$
Otherwise, there exists a constant $c_{g,K}>0$ such that for all $N\geq K$, $T\geq 1$,
$$
\mathfrak{R}(g,N,K,T)\geq c_{g,K}\min\{T,\sqrt{N^K T}\}.
$$ 
\label{lem:lower-bound-exchangeable}
\end{lemma}

%
%

Our proof is based on the following technical lemma:
\begin{lemma}\label{lem.existential}
Let $g\in C^m([0,b])$ be a real-valued and $m$-times continuously differentiable function on $[0,b]$, with $b\ge m$. Then the following two statements are equivalent: 
\begin{enumerate}
	\item $g$ is not a polynomial of degree at most $m-1$; 
	\item there exists a random vector $(X_1,\cdots,X_m)$ supported on $[0,1]^m$, which follows an exchangeable joint distribution $\mu$, and a scalar $x_0\in [0,1]$, such that 
	\begin{align}
	  \bE_\mu[g(X_1 + \cdots + X_{\ell-1} + (b-\ell+1)x_0 ) ] \nonumber\\
	  = \bE_\mu[g(X_1 + \cdots + X_{\ell} + (b-\ell)x_0) ]
	  \label{eq:equality}
	\end{align}
	for all $\ell = 1,2,\cdots,m-1$, and 
	\begin{align}
	\bE_\mu[g(X_1 + \cdots + X_{m-1} + (b-m+1)x_0) ]\nonumber\\
	 < \bE_\mu[g(X_1 + \cdots + X_m +(b-m)x_0) ]. 
	\label{eq:inequality}
	\end{align}
\end{enumerate}
\end{lemma}

The proof of Lemma \ref{lem.existential} is deferred to the Sec.~\ref{sec:proof-existential}. 
The construction of the distributions in Lemma \ref{lem.existential} is non-constructive and uses duality existential arguments.
Its proof also applies several technical tools from real analysis and functional analysis \citep{rudin1991functional,donoghue1969distributions,dudley2018real}.

We are now ready to prove Lemma \ref{lem:lower-bound-exchangeable}.
\begin{proof}
We first prove the case when $g(\cdot)$ is not a polynomial of degree at most $K-1$.
We use the construction $(X_1,\cdots,X_K)$ and $x_0$ in Lemma \ref{lem.existential} with $(m,b)=(K,K)$, and construct i.i.d.~copies $(X_{t,1},\cdots,X_{t,K})$ for each $t\in [T]$. Consider the following random strategy of the adversary $\mathcal A$ when the optimal subset is $S^\star$:
 at each time $t\in [T]$, nature assigns $(X_{t,1},\cdots,X_{t,K})$ to the restriction of $v_t$ to $S^\star$ with probability $\delta \in (0,1]$, and assigns $(x_0,\cdots,x_0)$ otherwise, with the parameter $\delta$ to be specified later; nature also assigns $v_{ti}=x_0$ for all $i\notin S^\star$.
 Suppose an algorithm selects subset $S_t\subseteq[N]$, $|S_t|=K$ at time $t$, such that $|S_t\cap S^\star|=\ell$. Then
 \begin{align}
 &\mathbb E_{v_t}[\mathbb E[r_t|S_t,v_t]] \nonumber\\
 &= \delta\mathbb E_{\mu}[g(X_1+\cdots+X_\ell+(K-\ell)x_0)] + (1-\delta)g(Kx_0)\nonumber\\
 &= g(Kx_0) + \delta\gamma\times \vct 1\{\ell = K\},\label{eq:obs-model}
 \end{align}
 where $\gamma := \mathbb E_{\mu}[g(X_1+\cdots+X_K)]-g(Kx_0) >0$.
 Essentially, Eq.~(\ref{eq:obs-model}) shows that the marginal distribution of $r_t$ conditioned on $S_t$ is $\mathrm{Bernoulli}(g(Kx_0))$ if $S_t\neq S^\star$, but $\mathrm{Bernoulli}(g(Kx_0)+\delta\gamma)$ if $S_t=S^\star$.
 
 Let $\mathbb P_{S}$ denote the distribution of $\{r_t\}_{t=1}^T$ when the adversary chooses $S^\star=S$ as the optimal subset.
 Let $\mathbb P_0$ also denote the distribution of $\{r_t\}_{t=1}^T$ with $v_{ti}\equiv x_0$.
As $\{r_t\}$ are binary,
 Eq.~(\ref{eq:obs-model}) implies that, for every $S\subseteq[N], |S|=K$, 
 \begin{equation}
 D_{\mathrm{KL}}(\mathbb P_0\|\mathbb P_S) \leq \mathbb E_0[T_S]\times \Gamma_{g,K}^2\delta^2,
 \label{eq:kl}
 \end{equation}
 where $\mathbb E_0$ is the expectation under $\mathbb P_0$, $T_S := \sum_{t=1}^T\vct 1\{S_t=S\}$ is the number of times $S$ is selected
 and constant $\Gamma_{g,K}<\infty$ depends only on $g(Kx_0)$ and $\mathbb E_\mu[g(X_1,\cdots,X_K)]$ and thus only on $g,K$. By Pinsker's inequality,
 \begin{align}
 \big|\mathbb E_0[T_S]-\mathbb E_S[T_S]\big|&\leq T \cdot {\mathrm{TV}}(\mathbb P_0,\mathbb P_S)\nonumber\\ 
 &\leq \Gamma_{g,K}\delta T\sqrt{\mathbb E_0[T_S]}.
 \label{eq:ts-diff}
 \end{align}
 Subsequently, 
 \begin{align}
& \mathfrak{R}(g,N,K,T) = \max_{|S^\star|=K}\mathbb E\left[\sum_{t=1}^T R(S^\star,v_t)-R(S_t,v_t)\right]\nonumber\\
&\geq \frac{1}{\binom{N}{K}}\sum_{|S^\star|=K}\mathbb E\left[\sum_{t=1}^T R(S^\star,v_t)-R(S_t,v_t)\right]\nonumber\\
&= \frac{\delta\gamma}{\binom{N}{K}}\sum_{|S^\star|=K}(T-\mathbb E_{S^\star}[T_{S^\star}])\nonumber\\
&\geq \frac{\delta\gamma}{\binom{N}{K}}\sum_{|S^\star|=K}(T-\mathbb E_{0}[T_{S^\star}] -\big|\mathbb E_{S^\star}[T_{S^\star}]-\mathbb E_0[T_{S^\star}]\big|)\nonumber\\
&= \frac{\delta\gamma}{\binom{N}{K}}\left[\binom{N}{K}T-T - \Gamma_{g,K}\delta T\sum_{|S|=K}\sqrt{\mathbb E_0[T_S]}\right]\label{eq:proof-lb-i1}\\
&\geq \frac{\delta\gamma}{\binom{N}{K}}\left[\binom{N}{K}T-T-\Gamma_{g,K}\delta T\sqrt{\binom{N}{K} T}\right]\label{eq:proof-lb-i2}.
 \end{align}
 Here, Eq.~(\ref{eq:proof-lb-i1}) holds because $\sum_{|S|=K}\mathbb E_0[T_S]= T$, and
 Eq.~(\ref{eq:proof-lb-i2}) is due to the Cauchy-Schwarz inequality.
 Setting $\delta = \min\{1,\sqrt{\binom{N}{K}}/(2\Gamma_{g,K}\sqrt{T})\}$, Eq.~(\ref{eq:proof-lb-i2}) is lower bounded by the minimum between $\Omega(T)$ and
 $$
 \frac{\delta\gamma}{\binom{N}{K}}\left[\binom{N}{K}T - T - \binom{N}{K}\frac{T}{2}\right] = \Omega_{g,K}\bigg(\sqrt{\binom{N}{K}T}\bigg),
 $$
 which is to be demonstrated. 
 
 The scenario when $g(\cdot)$ is a polynomial of degree $d<K$ can be proved in an entire similar way.
 Applying Lemma \ref{lem.existential} with $(m,b)=(d,K)$, we could obtain a random vector $(X_1,\cdots,X_d)$ and some $x_0\in [0,1]$ such that the conditions \eqref{eq:equality} and \eqref{eq:inequality} hold. The nature uses the same strategy, with the only difference being that the size of the random subset $S^\star$ is $d$ instead of $K$. In this case, any size-$K$ set $S$ with $S^\star\subseteq S$ gives the optimal reward, and the learner observes the non-informative outcome at time $t$ if and only if $S^\star\nsubseteq S_t$. Consequently, both Eqs.~(\ref{eq:obs-model},~\ref{eq:kl}) still hold, with $\vct 1\{\ell=K\}$ replaced by $\vct 1\{\ell\geq d\}$ in Eq.~(\ref{eq:obs-model}) and the definition of $T_{S}$ changed to $T_S = \sum_{t=1}^T \vct 1\{S\subseteq S_t\}$ in Eq.~(\ref{eq:kl}).
Using $$\sum_{|S|=d}T_{S}=\sum_{t=1}^T \sum_{|S|=d}\vct 1\{S\subseteq S_t\} = \binom{K}{d}T,$$
(\ref{eq:proof-lb-i2}) with $|S|=d$ yields a lower bound of $\mathfrak{R}(g,N,K,T)$:
\begin{align*}
\frac{\delta\gamma}{\binom{N}{d}}\left[\binom{N}{d}T-\binom{K}{d}T-\Gamma_{g,d,K}\delta T\sqrt{\binom{N}{d}\binom{K}{d} T} \right].
\end{align*}
Setting $\delta = \min\{1, \sqrt{\binom{N}{d})}/(2\Gamma_{g,d,K}\sqrt{\binom{K}{d}T})\}$ and noting that $\binom{K}{d}$ does not depend on $(N,T)$,
the above lower bound is simplified to
$$
\Omega_{g,d,K}\bigg(\sqrt{\binom{N}{d}T}\bigg),
$$
which is to be demonstrated.
\end{proof}

Next, we establish an $\Omega(\sqrt{N^KT})$ lower bound for the dynamic assortment optimization model described in Sec.~\ref{subsec:intro-mnl}.
It implies Corollary \ref{cor:main-result-informal-mnl-lb} in the introduction section.
\begin{lemma}
Consider the dynamic assortment optimization question with mutlinomial logit choice models described in Sec.~\ref{subsec:intro-mnl}.
Adopt the unit price model, with $p_i\equiv 1$ for all $i\in[N]$.
Then there exists an adversary choosing $\{v_{ti}\}_{t,i=1}^{T,N}$ and a constant $c_K>0$ depending only on $K$,
such that for any bandit algorithm and $N\geq K$, $T\geq 1$, it holds that
\begin{multline*}
\max_{|S^\star|=K}\mathbb E\left[\sum_{t=1}^T R(S^\star,v_t)-R(S_t,v_t)\right]\\
\geq c_K\min\{1,\sqrt{N^K T}\},
\end{multline*}
where $R(S,v_t) = \frac{\sum_{i\in S}p_iv_{ti}}{1+\sum_{i\in S}v_{ti}}$.
\label{lem:main-result-mnl-lb}
\end{lemma}
\begin{proof}
It is easy to verify that $R(S,v_t)=g(\sum_{i\in S}v_{ti})$ with $g(x)=x/(1+x)$.
Since the link function $g$ here is clearly not a polynomial of any degree, the construction in Lemma \ref{lem:lower-bound-exchangeable} should immediately imply an $\Omega(\sqrt{N^K T})$ lower bound.
However, in the multinomial logit choice model for dynamic assortment optimization, the bandit feedback is \emph{not} binary.
This means that expectation calculations in Eq.~(\ref{eq:obs-model}) are not sufficient,
and we have to calculate the KL-divergence between discrete observations directly.

Recall that in the MNL model, the bandit feedback $i_t$ is supported on $S_t\cup\{0\}$, with $\mathbb P[i_t=i|S_t,v_t]=v_{ti}/(1+\sum_{j\in S_t}v_{tj})$
for all $i\in S_t$ and $\mathbb P[i_t=0|S_t,v_t] = 1/(1+\sum_{j\in S_t}v_{tj})$.
Suppose $|S_t\cap S^\star|=\ell<K$. Then
\begin{align}
&\mathbb E_{v_t}[\mathbb P[i_t=0|S_t,v_t]] \nonumber\\
&= 1-\mathbb E_{\mu}[g(X_1+\cdots+X_{\ell}+(K-\ell)x_0)]\nonumber\\
&= 1-g(Kx_0) = 1/(1+Kx_0).\label{eq:proof-mnl-1}
\end{align}
For every $i\in S_t\backslash S^\star$, $v_{ti}\equiv x_0$, and therefore
\begin{align}
&\mathbb E_{v_t}[\mathbb P[i_t=i|S_t,v_t]]\nonumber\\
&= x_0(1-\mathbb E_{\mu}[g(X_1+\cdots+X_{\ell}+(K-\ell)x_0)])\nonumber\\
&= x_0(1-g(Kx_0)) = x_0/(1+Kx_0).\label{eq:proof-mnl-2}
\end{align}
Next consider any $i\in S_t\cap S^\star$. Because $X_1,\cdots,X_\ell$ are exchangeable, the probabilities $\mathbb E_{v_t}[\mathbb P[i_t=i|S_t,v_t]]$ are the same for all $i\in S_t\cap S^\star$. Define $\beta := \mathbb E_{v_t}[\mathbb P[i_t=i|S_t,v_t]]$ for some $i\in S_t\cap S^\star$. By the law of total probability and Eqs.~(\ref{eq:proof-mnl-1},\ref{eq:proof-mnl-2}) we have that
$$
1=\sum_{i\in S_t\cup\{0\}}\mathbb E_{v_t}[\mathbb P(i_t=i|S_t,v_t)] = \ell\beta + \frac{1+(K-\ell)x_0}{1+Kx_0}.
$$
Consequently, 
\begin{equation}
\beta = \frac{1}{\ell}\left[1-\frac{1+(K-\ell)x_0}{1+Kx_0}\right] = \frac{x_0}{1+Kx_0}.
\label{eq:proof-mnl-3}
\end{equation}

Comparing Eqs.~(\ref{eq:proof-mnl-1},\ref{eq:proof-mnl-2},\ref{eq:proof-mnl-3}), we conclude that for any $|S_t|=K$, $S_t\neq S^\star$, 
$\mathbb E_{v_t}[\mathbb P[i_t=0|S_t,v_t]] = 1/(1+Kx_0)$ and $\mathbb E_{v_t}[\mathbb P[i_t=i|S_t,v_t]]=x_0/(1+Kx_0)$ for all $i\in S_t$.
This shows that all $|S_t|=K$, $S_t\neq S^\star$ are information theoretically indistinguishable.
On the other hand, for $S_t=S^\star$, with probability $1-\delta$ all elements of $v_t$ are assigned $v_0$,
for which $\mathbb P[\cdot|S_t=S^\star,v_t]=\mathbb P[\cdot|S_t\neq S^\star,v_t]$.
Hence, $D_{\mathrm{KL}}(\mathbb P_0\|\mathbb P_{S^\star}) \leq O(\delta^2)\times \mathbb E_0[T_{S^\star}]$.
The rest of the proof is identical to the proof of Lemma \ref{lem:lower-bound-exchangeable} when the link function $g$ is not a polynomial.
\end{proof}

\section{Proof of Lemma \ref{lem.existential}}\label{sec:proof-existential}
We first prove the easy direction $2\Rightarrow 1$, whose contrapositive is that if $g$ is a polynomial of degree at most $m-1$, then \eqref{eq:equality} and \eqref{eq:inequality} cannot hold simultaneously. In fact, defining $s(x) :=g(x+bx_0) - g(bx_0)$ and $Y_i:= X_i-x_0$ for all $i\in [m]$, condition \eqref{eq:equality} implies that for all $\ell=1,\cdots,m-1$, it holds that $\bE[s(Y_1+\cdots+Y_\ell)] = 0$. By exchangeability of $(Y_1,\cdots,Y_m)$, this also shows that $\bE[s_{\ell}(Y_1,\cdots,Y_m)]=0$, where for all $\ell=1,2,\cdots,m-1$, 
\begin{align}\label{eq:symmetrized_polynomial}
s_\ell(Y_1,\cdots,Y_m) := \frac{1}{\binom{m}{\ell}}\sum_{S\subseteq [m]: |S|=\ell} s\left(\sum_{i\in S}Y_i \right). 
\end{align}
Since $s$ is a polynomial of degree at most $m-1$ and $s(0)=0$, the following lemma shows that $\bE[s_m(Y_1,\cdots,Y_m)]=0$, a contradiction to \eqref{eq:inequality}. 
\begin{lemma}\label{lemma:comb_identity}
For $m$ reals $y_1,\cdots,y_m$ and any polynomial $s$ of degree at most $m-1$, define $s_0\equiv s(0)$ and $s_\ell$ as in \eqref{eq:symmetrized_polynomial} for $\ell=1,\cdots,m$. Then the following identity holds:
\begin{align*}
\sum_{\ell=0}^m (-1)^{\ell}\binom{m}{\ell} s_\ell(y_1,\cdots,y_m) = 0.
\end{align*}
\end{lemma}
\begin{proof}
By linearity it suffices to prove Lemma \ref{lemma:comb_identity} for $s(x)=x^d$, $d\in \{0,1,\cdots,m-1\}$. In this case, simple algebra verifies that the coefficient of $\prod_{i=1}^m y_i^{r_i}$ with $r_i\ge 0, \sum_{i=1}^m r_i=d$ in $s_{\ell}(y_1,\cdots,y_m)$ is
$$
\frac{d!}{\prod_{i=1}^m (r_i!)}\cdot \frac{\binom{m-I(r)}{\ell - I(r)}}{\binom{m}{\ell}}, 
$$
where $I(r):= \sum_{i=1}^m \1(r_i>0)$, and $\binom{a}{b}:= 0$ if $b<0$. Consequently, the coefficient of $\prod_{i=1}^m y_i^{r_i}$ in the LHS of the equation is
\begin{align*}
&\frac{d!}{\prod_{i=1}^m (r_i!)} \sum_{\ell=0}^m (-1)^{\ell}\binom{m-I(r)}{\ell-I(r)} \\
&=\frac{d!}{\prod_{i=1}^m (r_i!)} \sum_{\ell=I(r)}^m (-1)^{\ell}\binom{m-I(r)}{\ell-I(r)} =0, 
\end{align*}
where the last identity makes use of the inequality $I(r)\le d<m$. 
\end{proof}

Next we prove the hard direction $1\Rightarrow 2$. The proof makes use of the idea in convex analysis: after fixing $x_0$, the problem is to find an exchangeable distribution of $(X_1,\cdots,X_m)$ from the convex set of all such distributions; moreover, both constraints \eqref{eq:equality} and \eqref{eq:inequality} are linear in the joint distribution, and the Dirac point measure on $(x_0,\cdots,x_0)$ satisfies \eqref{eq:equality} as well as \eqref{eq:inequality} if $>$ is replaced by $\ge$. Therefore, there are two convex sets in total, one being the family of all exchangeable joint distributions and one from the constraints \eqref{eq:equality} and \eqref{eq:inequality}, and their closure has an intersection point, i.e. the Dirac point measure at $(x_0,\cdots,x_0)$. Now our target is to show that these sets have a non-empty intersection \emph{without} taking the closure, and the following lemma characterizes a sufficient condition via duality. 

\begin{lemma}\label{lemma:duality}
For $g\in C([0,b])$ and a fixed $x_0\in [0,1]$, there exists an exchangeable Borel distribution on $[0,1]^m$ such that both \eqref{eq:equality} and \eqref{eq:inequality} hold, if the following condition holds: for every non-zero vector $\lambda = (\lambda_1,\cdots,\lambda_m)$ with $\lambda_m\ge 0$, and the functions $g_\ell(x_1,\cdots,x_m)$ defined as
	\begin{align*} 
	g_\ell(x_1,\cdots,x_m) = \frac{1}{\binom{m}{\ell}}\sum_{S\subseteq[m]: |S|=\ell} g\left(\sum_{i\in S}x_i + (b-\ell)x_0\right) 
	\end{align*}
	for $\ell=1,2,\cdots,m$,
	the point $(x_0,\cdots,x_0)$ is not a global maxima of the function $\sum_{\ell=1}^m \lambda_\ell g_\ell(x)$ over $[0,1]^m$.
\end{lemma}
\begin{proof}
We prove the contrapositive of this result. Let $X$ be the topological vector space of all finite Borel signed measures on $[0,1]^m$ (which are also Radon measures by Ulam's theorem; cf. \citep[Theorem 7.1.4]{dudley2018real}) equipped with the weak-$^\star$ topology (as the dual of $C([0,1]^m)$), and $A\subseteq X$ be the collection of all exchangeable Borel probability measures (i.e. invariant to permutations). Moreover, for $\varepsilon>0$, we define $B_\varepsilon \subseteq X$ as the collection of all signed Borel measures $\mu$ such that
$$\int_{[0,1]^m} (g_\ell(x) - g(bx_0)) \mu(\ud x) = 0$$ for $\ell=1,\cdots,m-1$ and
$$\int_{[0,1]^m} (g_m(x) - g(bx_0) - \varepsilon) \mu(\ud x) = 0.$$
Therefore, the non-existence of such an exchangeable distribution implies that $A\cap B_\varepsilon=\emptyset$ for all $\varepsilon>0$. Since $[0,1]^m$ is compact, the set $A$ is a closed subset of the unit weak-$^\star$ ball, and is therefore compact due to Banach--Alaoglu (see, e.g. \citep[Theorem 3.15]{rudin1991functional}). Moreover, as $g_\ell\in C([0,1]^m)$, the set $B_\varepsilon$ is closed under the weak-$^\star$ topology. Finally, as both sets are convex, and \citep[Theorem 3.10]{rudin1991functional} shows that the dual of $X$ under the weak-$^\star$ topology is $C([0,1]^m)$, \citep[Theorem 3.4]{rudin1991functional} implies that there exist some $f_\varepsilon\in C([0,1]^m)$ and $\gamma_\varepsilon\in \bR$ such that
\begin{equation}\label{eq:continuous_functional}
\sup_{\mu\in A} \int f_\varepsilon \ud\mu < \gamma_\varepsilon \le \inf_{\nu\in B_\varepsilon} \int f_\varepsilon \ud\nu. 
\end{equation}
As $B_\varepsilon$ is a linear subspace of $X$, the RHS of \eqref{eq:continuous_functional} must be zero (otherwise it would be $-\infty$). Then by \citep[Lemma 3.9]{rudin1991functional}, we must have $f_\varepsilon(x) = \sum_{\ell=1}^{m-1} \lambda_{\varepsilon,\ell}(g_\ell(x) - g(bx_0))+\lambda_{\varepsilon,m}(g_m(x) - g(bx_0) - \varepsilon)$ for some vector $\lambda_\varepsilon=(\lambda_{\varepsilon,1},\cdots,\lambda_{\varepsilon,m})$. Plugging this back into the first inequality of \eqref{eq:continuous_functional}, and defining $\mu_0\in A$ as the Dirac point measure on $(x_0,\cdots,x_0)$, we arrive at
\begin{equation}\label{eq:dual}
\sup_{\mu\in A} \int \left(\sum_{\ell=1}^{m} \lambda_{\varepsilon,\ell} g_{\ell}(x)\right)(\mu(\ud x) - \mu_0(\ud x)) < \lambda_{\varepsilon,m}\varepsilon. 
\end{equation}
Since $g_\varepsilon(x) = \sum_{\ell=1}^{m} \lambda_{\varepsilon,\ell} g_{\ell}(x)$ is a symmetric function (i.e. invariant to permutations of the input), the LHS of \eqref{eq:dual} is simply $\max_{x\in [0,1]^m} g_\varepsilon(x) - g_\varepsilon(x_0,\cdots,x_0)\ge 0$. Then $\lambda_{\varepsilon,m}>0$, and by multiplying a positive constant to all entries of $\lambda_\varepsilon$, we may assume that $\max_\ell |\lambda_{\varepsilon,\ell}|=1$ and $\max_{x\in [0,1]^m} g_\varepsilon(x) - g_\varepsilon(x_0,\cdots,x_0) < \varepsilon$. Choosing $\varepsilon_n\to 0$, the compactness of $[-1,1]^m\ni \lambda_{\varepsilon_n}$ implies some subsequence $\lambda_{\varepsilon_{n_k}}\to \lambda$ as $k\to\infty$. Taking the limit along this subsequence, it is clear that $\lambda$ is a non-zero vector with $\lambda_m\ge 0$, and $(x_0,\cdots,x_0)$ is a global maxima of the function $g(x)=\sum_{\ell=1}^m \lambda_\ell g_\ell(x)$. 
\end{proof}

Now it remains to choose a suitable $x_0\in [0,1]$ such that the condition in Lemma \ref{lemma:duality} holds. We choose $x_0$ to be any point in $(0,1)$ such that $g^{(\ell)}(bx_0)\neq 0$ for all $\ell\in [m]$, whose existence is ensured by the following lemma, which itself is a standard exercise on the Baire category theorem. 
\begin{lemma}[A slight variant of \citep{donoghue1969distributions}, Page 53] \label{lemma:baire}
If $g\in C^m$ satisfies $g^{(n_x)}(x)=0$ for some $n_x\in [m]$ and every $x\in (0,1)$, then $g$ is a polynomial of degree at most $m-1$ on $(0,1)$. 
\end{lemma}

Now we assume by contradiction that the function $f(x):=\sum_{\ell=1}^n \lambda_\ell g_\ell(x)$ attains its maximum at $x=(x_0,\cdots,x_0)$ for some non-zero vector $\lambda$. We will prove by induction on $k=1,2,\cdots,m$ the following claims: 
\begin{enumerate}
	\item $D^\alpha f(x_0,\cdots,x_0)=0$ for all multi-indices $\alpha=(\alpha_1,\cdots,\alpha_m)\in \naturals^m$ with $|\alpha|:= \sum_{\ell=1}^k \alpha_\ell   = k$; 
	\item $\sum_{\ell=1}^m \lambda_\ell\cdot \binom{m-k}{\ell-k}  / \binom{m}{\ell} = 0$, with $\binom{0}{0}:= 1$ and $\binom{a}{b}:=0$ for $b<0$. 
\end{enumerate}
For the base case $k=1$, the first claim is exactly the first-order condition for the maxima, and the second claim follows from the first claim, the identity $\nabla g_\ell(x_0,\cdots,x_0) = g'(bx_0){\vct 1}\cdot \binom{m-1}{\ell-1}/\binom{m}{\ell}$, and our choice of $x_0$ that $g'(bx_0)\neq 0$. Now suppose that these claims hold for $1,2,\cdots,k-1$. Then for $|\alpha|=k$ with $I(\alpha):= \sum_{\ell=1}^m \1(\alpha_\ell>0) \le k-1$, simple algebra gives that
\begin{align*}
D^\alpha f(x_0,\cdots,x_0) = g^{(k)}(bx_0)\cdot \sum_{\ell=1}^m \lambda_\ell \frac{\binom{m-I(\alpha)}{\ell-I(\alpha)}}{\binom{m}{\ell}} = 0, 
\end{align*}
where the last identity is due to the inductive hypothesis for $k' = I(\alpha) < k$. Therefore, for the first claim it remains to consider multi-indices $\alpha$ with $|\alpha|=I(\alpha)=k$. By the inductive hypothesis and the Taylor expansion for multivariate functions, for $\delta>0$ sufficiently small and every binary vector $(\varepsilon_1,\cdots,\varepsilon_m)\in \{\pm 1\}^m$, it holds that
\begin{align*}
&f(x_0+\delta \varepsilon_1,\cdots,x_0+\delta \varepsilon_m) = f(x_0,\cdots,x_0) \\
&+ g^{(k)}(bx_0) \sum_{\ell=1}^m \lambda_\ell \frac{\binom{m-k}{\ell-k}}{\binom{m}{\ell}}  \delta^k\sum_{S\subseteq[m]: |S|=k}\prod_{i\in S}\varepsilon_i + o(\delta^k). 
\end{align*}
As $\bE[\sum_{S\subseteq[m]: |S|=k}\prod_{i\in S}\varepsilon_i] = 0$ when $\varepsilon_1,\cdots,\varepsilon_m$ are Radamacher random variables, and the term inside the expectation is not identically zero, we conclude that this term could be either positive or negative after carefully choosing $(\varepsilon_1,\cdots,\varepsilon_m)$. As a result, in order for $(x_0,\cdots,x_0)$ to be a maxima of $f$, the above expansion implies that $\sum_{\ell=1}^m \lambda_\ell\cdot \binom{m-k}{\ell-k}  / \binom{m}{\ell} = 0$ (recall that $g^{(k)}(bx_0)\neq 0$), which is exactly the second claim for $k$. The remainder of the first claim also follows from the second claim, and the induction is complete. 

Finally we derive a contradiction from the above result, and thereby show that such a non-zero vector $\lambda$ cannot exist. In fact, the second claim of the inductive result for $k=1,2,\cdots,m$ constitutes a linear system $A\lambda=0$ for the vector $\lambda$, where $A$ is a upper triangular matrix with non-zero diagonal entries. Therefore, we have $\lambda=0$, which is a contradiction.

\section{Conclusion and future directions}

In this paper study the adversarial combinatorial bandit problem with general reward functions $g$,
with a complete characterization of the minimax regret depending on whether $g$ is a low-degree polynomial.
For the most general case when $g$ is \emph{not} a polynomial, including dynamic assortment optimization under the multinomial logit choice models,
our results imply an $\Omega_K(\sqrt{N^K T})$ regret lower bound, hinting that it is not possible for any bandit algorithm to exploit inherent correlation among subsets/assortments.
We believe it is a promising future research direction to study models that interpolate between the stochastic and the adversarial settings,
in order to achieve intermediate regret bounds.

When $g$ is a general non-linear function, the adversarial construction in our lower bound proof can be interpreted as a \emph{latent variable model}: 
first sample $v_t\sim\mu$, and then sample $r_t\sim\mathbb P(\cdot|S_t,v_t)$.
To foster identifiability, one common assumption is to have access to multiple observations $\{r_t^1,\cdots,r_t^M\}$ conditioned on the same latent variable $v_t$,
such as ``high dimensionality'' assumptions in learning Gaussian mixture models \citep{ge2015learning} and learning topic models with multiple words per document \citep{anandkumar2015spectral}.
This motivates the following model: at the beginning of each $\tau\in\{1,2,\cdots,T/M\}$ epoch, an adaptive adversary chooses $v_\tau\in[0,1]^N$;
afterwards, the bandit algorithm produces subsets $\{S_\tau^1,\cdots,S_\tau^M\}\subseteq[N]$ and observes feedback $r_\tau^t\sim \mathbb P(\cdot|S_\tau^t,v_\tau)$ for $t=1,2,\cdots,M$.
Clearly, with $M=T$ we recover the stochastic (stationary) setting in which $\{v_t\}$ do not change, and with $M=1$ we recover the adversarial setting in which $v_t$ is different for each time period. 
An intermediate value of $1<M<T$ is likely to result in interesting interpolation between the two settings, 
and we leave this as an interesting future research question.
\bibliography{di}
\bibliographystyle{icml2021}

\appendix


\end{document}